\documentclass{article}

\usepackage{algorithmic,algorithm2e}
\usepackage{amssymb,amsmath,amsthm}

\def\beq{\begin{equation}}    \def\eeq{\end{equation}}
\def\beqn{\begin{displaymath}}\def\eeqn{\end{displaymath}}
\def\bqa{\begin{eqnarray}}    \def\eqa{\end{eqnarray}}
\def\bqan{\begin{eqnarray*}}  \def\eqan{\end{eqnarray*}}

\def\eps{\varepsilon}
\def\epstr{\epsilon}                    

\def\fr#1#2{{\textstyle{#1\over#2}}}

\def\E{{\mathbb E}}             
\def\M{{\cal M}}
\def\w{\omega}                  
\def\po{{\pi^\circ}}            
\def\ho{h^\circ}                

\def\g{\gamma}
\def\z{z}

\def\A{{\cal A}}
\def\B{{\cal B}}
\def\O{{\cal O}}
\def\R{{\cal R}}
\def\H{{\cal H}}
\theoremstyle{plain}
\newtheorem{theorem}{Theorem}
\newtheorem{lemma}[theorem]{Lemma}

\newtheorem{definition}[theorem]{Definition}

\newtheorem{remark}[theorem]{Remark}

\newenvironment{keywords}{\centerline{\bf\small
Keywords}\begin{quote}\small}{\par\end{quote}\vskip 1ex}

\def\argmax{\mathop{\arg\max}}
\newdimen\paravsp  \paravsp=1.3ex
\def\paradot#1{\vspace{\paravsp plus 0.5\paravsp minus 0.5\paravsp}\noindent{\bf\boldmath{#1.}}}

\begin{document}

\title{\bf\Large\hrule height5pt \vskip 4mm
Optimistic Agents are Asymptotically Optimal
\vskip 4mm \hrule height2pt}

\author{{\bf Peter Sunehag} and {\bf Marcus Hutter}\\
\normalsize peter.sunehag@anu.edu.au,\ marcus.hutter@anu.edu.au\\
\normalsize Research School of Computer Science \\[-0.5ex]
\normalsize Australian National University \\[-0.5ex]
\normalsize Canberra, ACT, 0200, Australia}

\date{September 2012}

\maketitle

\begin{abstract}
We use optimism to introduce generic asymptotically optimal
reinforcement learning agents. They achieve, with an arbitrary
finite or compact class of environments, asymptotically optimal
behavior. Furthermore, in the finite deterministic case we provide
finite error bounds.
\end{abstract}


\section{Introduction}

This article studies a fundamental question in artificial
intelligence; given a set of environments, how do we define an agent
that eventually acts optimally regardless of which of the
environments it is in. This question relates to the even more
fundamental question of what intelligence is. \cite{MH05} defines an
intelligent agent as one that can act well in a large range of
environments. He studies arbitrary classes of environments with
particular attention to universal classes of environments like all
computable (deterministic) environments and all lower
semi-computable (stochastic) environments. He defines the AIXI agent
as a Bayesian reinforcement learning agent with a universal
hypothesis class and a Solomonoff prior. This agent has some
interesting optimality properties. Besides maximizing expected
utility with respect to the a priori distribution by design, it is
also Pareto optimal and self-optimizing when this is possible for
the considered class. It was, however, shown in \cite{Orseau10} that
it is not guaranteed to be asymptotically optimal for all computable
 (deterministic) environments. \cite{Lattimore11} shows that this is not surprising
since, at least for geometric discounting, no agent can be.
\cite{Lattimore11} also shows that in a weaker (in average) sense,
optimality can be achieved for the class of all computable
environments using an algorithm that includes long exploration
phases. Furthermore, it is simple to realize that Bayesian agents
do not always achieve optimality for a finite class of
deterministic environments even if all prior weights are strictly
positive.

We use the principle of optimism to define an agent that for any
finite class of deterministic environments, eventually acts
optimally. We extend our results to the case of finite and compact
classes of stochastic environments. In the deterministic case we
also prove finite error bounds. Optimism has previously been used to
design exploration strategies for both discounted and undiscounted
MDPs \cite{KS98,strehl05,auer06,Tor12}, though here we define
optimistic algorithms for any finite class of environments.

\paradot{Related work}
Besides AIXI \cite{MH05} that was discussed above,
\cite{Lattimore11} introduces an agent which achieves asymptotic
optimality in an average sense for the class of all deterministic
computable environments. There is, however, no time step after which
it is optimal at every time step. This is due to an infinite number
of long exploration phases. We introduce an agent, that for finite
classes of environments, does eventually achieve optimality for
every time step. For the stochastic case, the agent achieves with
any given probability, optimality within $\epsilon$ for any
$\epsilon>0$. Our very simple agent is relying elegantly on the
principle of optimism, used previously in the restrictive MDP case
with discounting \cite{KS98,strehl05,Tor12} and without
\cite{auer06}, instead of an indefinite number of explicitly
enforced bursts of exploration. \cite{RH08} also introduces an agent
that relies on bursts of exploration with the aim of achieving
asymptotic optimality. The asymptotic optimality guarantees are
restricted to a setting where all environments satisfy a certain
restrictive value-preservation property. \cite{DKM05} studied
learning general Partially Observable Markov Decision Processes
(POMDPs). Though POMDPs constitute a very general reinforcement
learning setting, we are interested in agents that can be given any
(deterministic or stochastic) class of environments and successfully
utilize the knowledge that the true environment lies in this class.

\paradot{Background}
We will consider an agent \cite{RN10,MH05} that interacts with an
environment through performing actions $a_t$ from a finite set $\A$
and receives observations $o_t$ from a finite set $\O$ and rewards
$r_t$ from a finite set $\R\subset[0,1]$. Let $\H= (\A\times\O\times
R)^*$ be the set of histories and $R:\H\to\mathbb{R}$ the return
$$R(a_1o_1r_1a_2o_2r_2...a_no_nr_n)=\sum_{j=1}^n r_j\g^j$$ with the
obvious extension to infinite sequences. A function from
$\H\times\A$ to $\O\times\R$ is called a deterministic environment
(studied in Section \ref{finitedet}. A function $\pi:\H\to\A$ is
called a policy or an agent. We define the value function $V$ by
$V^\pi_\nu(h_{t-1}):=R(h_{t:\infty})=\sum_{i=t}^\infty
\gamma^{i-t}r_i$ where the sequence $r_i$ are the rewards achieved
by following $\pi$ from time step $t$ onwards in environment $\nu$
after having seen $h_{t-1}$.

Instead of viewing the environment as a function from $\H\times\A$
to $\O\times\R$ we can equivalently write it as a function
$\nu:\H\times\A\times\O\times\R\to\{0,1\}$ where we write
$\nu(o,r|h,a)$ for the function value of $(h,a,o,r)$. It equals zero
if in the first formulation $(h,a)$ is not sent to $(o,r)$ and $1$
if it is. In the case of stochastic environments, which we will
study in Section \ref{finitesto}, we instead have a function
$\nu:\H\times\A\times\O\times\R\to[0,1]$ such that
$\sum_{o,r}\nu(o,r|h,a)=1\ \forall h,a$.  Furthermore, we define
$\nu(h_t|\pi):=\nu(or_{1:t}|\pi):=\Pi_{i=1}^t
\nu(o_ir_i|a_i,h_{i-1})$ where $a_i=\pi(h_{i-1})$. $\nu(\cdot|\pi)$
is a probability measure over strings or sequences as will be
discussed in the next section and we can define
$\nu(\cdot|\pi,h_{t-1})$ by conditioning $\nu(\cdot|\pi)$ on
$h_{t-1}$. We define
$V^\pi_\nu(h_{t-1}):=\E_{\nu(\cdot|\pi,h_{t-1})} R(h_{t:\infty})$ as
the $\nu$-expected return of policy $\pi$.

A special case of an environment is a Markov Decision Process (MDP)
\cite{SB98}. This is the classical setting for reinforcement
learning. In this case the environment does not depend on the full
history but only on the latest observation and action and is,
therefore, a function from $\O\times\A\times\O\times\R$ to $[0,1]$.
In this situation one often refers to the observations as states
since the latest observation tells us everything we need to know. In
this situation, there is an optimal policy that can be represented
as a function from the state set $\mathcal{S}$ (:=$\O$) to $\A$. We
only need to base our decision on the latest observation. Several
algorithms \cite{KS98,strehl05,Tor12} have been devised for solving
discounted ($\gamma<1$) MDPs for which one can prove PAC (Probably
Approximately Correct) bounds. They are finite time bounds that hold
with high probability and depend only polynomially on the number of
states, actions and the discount factor. These methods are relying
on optimism as the method for making the agent sufficiently
explorative. Optimism roughly means that one has high expectations
for what one does not yet know. Optimism was also used to prove
regret bounds for undiscounted ($\gamma=1$) MDPs in \cite{auer06}
which was extended to feature MDPs in \cite{MR11}. Note that these
methods are restricted to MDPs and that we do not make any (Markov,
ergodicity, stationarity, etc.) assumptions on the environments,
only on the size of the class.

\paradot{Outline}
In this article we will define optimistic agents in a far more
general setting than MDPs and prove asymptotic optimality results.
The question of their mere existence is already non-trivial, hence
asymptotic results deserve attention. In Section \ref{finitedet} we
consider finite classes of deterministic environments and introduce
a simple optimistic agent that is guaranteed to eventually act
optimally. We also provide finite error bounds. In Section
\ref{finitesto} we generalize to finite classes of stochastic
environments and in Section \ref{compact} to compact classes.

\section{Finite Classes of Deterministic Environments}\label{finitedet}

Given a finite class of deterministic environments
$\M=\{\nu_1,...,\nu_m\}$, we define an algorithm that for any
unknown environment from $\M$ eventually achieves optimal behavior
in the sense that there exists $T$ such that maximum reward is
achieved from time $T$ onwards. The algorithm chooses an optimistic
hypothesis from $\M$ in the sense that it picks the environment in
which one can achieve the highest reward (in case of a tie, choose
the environment which comes first in an enumeration of $\M$) and
then the policy that is optimal for this environment is followed. If
this hypothesis is contradicted by the feedback from the
environment, a new optimistic hypothesis is picked from the
environments that are still consistent with $h$. This technique has
the important consequence that if the hypothesis is not contradicted
we are still acting optimally when optimizing for this incorrect
hypothesis.

\begin{algorithm}[!h]
\caption{Optimistic Agent ($\po$) for Deterministic
Environments}\label{algdet}
\begin{algorithmic}[1]
\REQUIRE Finite class of deterministic environments $\M_0\equiv\M$
\STATE $t=1$ \REPEAT
  \STATE $(\pi^*,\nu^*)\in\argmax_{\pi\in\Pi,\nu\in\M_{t-1}}V^\pi_\nu(h_{t-1})$
  \REPEAT
  \STATE $a_t=\pi^*(h_{t-1})$
  \STATE Perceive $o_t r_t$ from environment $\mu$
  \STATE $h_t\leftarrow h_{t-1} a_t o_t r_t$
  \STATE Remove all inconsistent environments from $\M_t$ ($\M_t:=\{\nu\in\M_{t-1}:h_t^{\po,\nu}= h_t\}$)
  \STATE $t\leftarrow t+1$
  \UNTIL $\nu^*\not\in\M_{t-1}$
\UNTIL $\M$ is empty
\end{algorithmic}
\end{algorithm}

Let $h_t^{\pi,\nu}$ be the history up to time $t$ generated by
policy $\pi$ in environment $\nu$. In particular let
$\ho:=h^{\po,\mu}$ be the history generated by Algorithm 1 (policy
$\po$) interacting with the actual ``true'' environment $\mu$. At
the end of cycle $t$ we know $\ho_t=h_t$. An environment $\nu$ is
called consistent with $h_t$ if $h_t^{\po,\nu}=h_t$. Let $\M_t$ be
the environments consistent with $h_t$. The algorithm only needs to
check whether $o_t^{\po,\nu}=o_t$ and $r_t^{\po,\nu}=r_t$ for each
$\nu\in\M_{t-1}$, since previous cycles ensure
$h_{t-1}^{\po,\nu}=h_{t-1}$ and trivially $a_t^{\po,\nu}=a_t$. The
maximization in Algorithm 1 that defines optimism at time $t$ is
performed over all $\nu\in\M_t$, the set of consistent hypotheses at
time $t$, and $\pi\in\Pi=\Pi^{all}$ is the class of {\em all}
deterministic policies.

\begin{theorem}[Optimality, Finite Deterministic Class]\label{thm:det}
If we use Algorithm 1 ($\po$) in an environment $\mu\in\M$ , then
there is $T<\infty$ such that
$$
  V^\po_\mu(h_t) \;=\; \max_\pi V^\pi_\mu(h_t)\ \forall t\geq T.
$$

\end{theorem}
A key to proving Theorem \ref{thm:det} is time-consistency
\cite{Tor11} of geometric discounting. The following lemma tells us
that if we act optimally with respect to a chosen optimistic
hypothesis, it remains optimistic until contradicted.

\begin{lemma}[Time-consistency]\label{lemma:time-con}
Suppose
$(\pi^*,\nu^*)\in\argmax_{\pi\in\Pi,\nu\in\M_{t}}V^\pi_\nu(h_{t})$,
that we act according to $\pi^*$ from time $t$ to time $\tilde{t}-1$
and that $\nu^*$ is still consistent at time $\tilde{t}>t$ , then
$(\pi^*,\nu^*)\in\argmax_{\pi\in\Pi,\nu\in\M_{\tilde{t}}}V^\pi_\nu(h_{\tilde{t}})$.
\end{lemma}

\begin{proof}
Suppose that
$V^{\pi^*}_{\nu^*}(h_{\tilde{t}})<V^{\tilde{\pi}}_{\tilde{\nu}}(h_{\tilde{t}})$
for some $\tilde{\pi}$, $\tilde{\nu}$. It holds that
$V^{\pi^*}_{\nu^*}(h_{t})=C+\gamma^{\tilde{t}-t}
V^{\pi^*}_{\nu^*}(h_{\tilde{t}})$ where $C$ is the accumulated
reward between $t$ and $\tilde{t}-1$. Let $\hat{\pi}$ be a policy
that equals $\pi^*$ from $t$ to $\tilde{t}-1$ and then equals
$\tilde{\pi}$. It follows that
$V^{\hat{\pi}}_{\tilde{\nu}}(h_{t})=C+\gamma^{\tilde{t}-t}
V^{\hat{\pi}}_{\tilde{\nu}}(h_{\tilde{t}})>C+\gamma^{\tilde{t}-t}
V^{\pi^*}_{\nu^*}(h_{\tilde{t}})=V^{\pi^*}_{\nu^*}(h_{t})$ which
contradicts the assumption
$(\pi^*,\nu^*)\in\argmax_{\pi\in\Pi,\nu\in\M_{t}}V^\pi_\nu(h_{t})$.
Therefore, $V^{\pi^*}_{\nu^*}(h_{\tilde{t}})\geq
V^{\tilde{\pi}}_{\tilde{\nu}}(h_{\tilde{t}})$ for all $\tilde{\pi}$,
$\tilde{\nu}$.
\end{proof}

\begin{proof}{\bf (Theorem \ref{thm:det})}
At time $t$ we know $h_t$. If some $\nu\in\M_{t-1}$ is
inconsistent with $h_t$, i.e.\ $h_t^{\po,\nu}\neq h_t$, it gets
removed, i.e.\ is not in $\M_{t'}$ for all $t'\geq t$.

Since $\M_0=\M$ is finite, such inconsistencies can only happen
finitely often, i.e.\ from some $T$ onwards we have $\M_t=\M_\infty$
for all $t\geq T$. Since $h_t^{\po,\mu}=h_t\ \forall t$, we know
that $\mu\in\M_t\ \forall t$.

Assume $t\geq T$ henceforth. The optimistic hypothesis will not
change after this point. If the optimistic hypothesis is the true
environment $\mu$, we have obviously chosen the true optimal policy.

In general, the optimistic hypothesis $\nu^*$ is such that it will
never be contradicted while actions are taken according to $\po$,
hence $(\pi^*,\nu^*)$ do not change anymore. This implies
$$
  V^\po_\mu(h_t)
  \;=\; V^{\pi^*}_\mu(h_t)
  \;=\; V^{\pi^*}_{\nu^*}(h_t)
  \;=\; \max_{\nu\in\M_t} \max_{\pi\in\Pi} V^\pi_\nu(h_t)
  \;\geq\; \max_{\pi\in\Pi} V^\pi_\mu(h_t)
$$
for all $t\geq T$. 
The first equality follows from $\po$ equals $\pi^*$ from $t\geq T$ onwards. %
The second equality follows from consistency of $\nu^*$ with $\ho_{1:\infty}$. %
The third equality follows from optimism, %
the constancy of $\pi^*$, $\nu^*$, and $\M_t$ for $t\geq T$, %
and time-consistency of geometric discounting (Lemma \ref{lemma:time-con}). %
The last inequality follows from $\mu\in\M_t$. %
The reverse inequality $V^{\pi^*}_\mu(h_t)\leq \max_\pi
V^\pi_\mu(h_t)$
follows from $\pi^*\in\Pi$. %
Therefore $\po$ is acting optimally at all times $t\geq T$.
\end{proof}

Besides the eventual optimality guarantee above, we also provide a
bound on the number of time steps for which the value of following
Algorithm 1 is more than a certain $\varepsilon>0$ less than
optimal. The reason this bound is true is that we only have such suboptimality for a certain number of time steps before a point where the current hypothesis becomes inconsistent and the number of such inconsistency points are bounded by the number of environments.

\begin{theorem}[Finite error bound]\label{thm:fin}
Following $\po$ (Algorithm 1), $$V^\po_\mu(h_t) \geq
\max_{\pi\in\Pi} V^\pi_\mu(h_t) - \eps,\ 0<\eps<1/(1-\g)$$ for all
but at most $|\M|{\log\eps(1-\g)\over \g-1}$ time steps $t$.
\end{theorem}

\begin{proof}
Consider the $\ell$-truncated value
$$
  V^\pi_{\nu,\ell}(h_t) \;:=\; \sum_{i=t+1}^{t+\ell}\gamma^{i-t-1}r_i
$$
where the sequence $r_i$ are the rewards achieved by following $\pi$
from time $t+1$ to $t+\ell$ in $\nu$ after seeing $h_t$. By letting
$\ell=\frac{\log{\eps(1-\gamma)}}{\log{\gamma}}$ (which is positive
due to negativity of both numerator and denominator)  we achieve $
|V^\pi_{\nu,\ell}(h_t)-
V^\pi_{\nu}(h_t)|\leq\frac{\gamma^l}{1-\gamma}=\epsilon$. Let
$(\pi_t^*,\nu_t^*)$ be the policy-environment pair selected by
Algorithm 2 in cycle $t$.

Let us first assume $h_{t+1:t+\ell}^{\po,\mu} =
h_{t+1:t+\ell}^{\po,\nu_t^*}$, i.e.\ $\nu_t^*$ is consistent with
$\ho_{t+1:t+\ell}$, and hence $\pi_t^*$ and $\nu_t^*$ do not change
from $t+1,...,t+\ell$ (inner loop of Algorithm 1). Then
$$
  V^\po_\mu(h_t)
  \stackrel{\makebox[10ex]{\footnotesize drop terms,~~~~~~}\atop\scriptstyle\downarrow}{\geq}
  V^\po_{\mu,\ell}(h_t)
  \stackrel{\makebox[10ex]{\footnotesize same $h_{t+1:t+\ell}$,~~~~~~~}\atop\textstyle\downarrow}{=}
  V^\po_{\nu_t^*,\ell}(h_t)
  \stackrel{\makebox[10ex]{\footnotesize $\po\!=\!\pi_t^*$ on $h_{t+1:t+\ell}$,~}\atop\textstyle\downarrow}{=}
  V^{\pi_t^*}_{\nu_t^*,\ell}(h_t)
$$ \vspace{-3ex}
$$
  \mathop{\geq}\limits_{\scriptstyle\uparrow\atop\makebox[6ex]{\footnotesize ~~~~~~~~~~bound extra terms}}
  V^{\pi_t^*}_{\nu_t^*}(h_t) - \fr{\g^\ell}{1-\g}
  \mathop{=}\limits_{\textstyle\uparrow\atop\makebox[7ex]{\footnotesize ~~~~~~~~~~~~~~~~~~~~~~~def.\ of $(\pi_t^*,\nu_t^*)$ and $\smash{\eps:={\g^\ell\over 1-\g}}$}}
  \max_{\nu\in\M_t}\max_{\pi\in\Pi} V^\pi_\nu(h_t) - \eps
  \mathop{\geq}\limits_{\scriptstyle\uparrow\atop\makebox[7ex]{\footnotesize $\mu\in\M_t$}}
  \max_{\pi\in\Pi} V^\pi_\mu(h_t) - \eps.
$$

Now let $t_1,...,t_K$ be the times $t$ at which the currently
selected $\nu_t^*$ gets inconsistent with $h_t$, i.e.\
$\{t_1,...,t_K\}=\{t:\nu_t^*\not\in\M_t\}$.
Therefore $\ho_{t+1:t+\ell}\neq h_{t+1:t+\ell}^{\po,\nu_t^*}$ (only)
at times $t\in{{\cal
T}_{\!\times}}:=\:\bigcup_{i=1}^K\{t_i-\ell,...,t_i-1\}$, which
implies $V^\po_\mu(h_t) \geq \max_{\pi\in\Pi} V^\pi_\mu(h_t) - \eps$
except possibly for $t\in{\cal T}_{\!\times}$. Finally
$$
  |{\cal T}_{\!\times}|
  \;=\; \ell\!\cdot\!K
  \;<\; \ell\!\cdot\!|\M|
  \;=\; {\log\eps(1-\g)\over\log\g}|\M|\;\leq\; |\M|{\log\eps(1-\g)\over
  \g-1}
$$

\end{proof}

We refer to the algorithm above as the conservative agent since it
sticks to its model for as long as it can. The corresponding liberal
agent reevaluates its optimistic hypothesis at every time step and
can switch between different optimistic policies at any time.
Algorithm 1 is actually a special case of this as shown by Lemma
\ref{lemma:time-con}. The liberal agent is really a class of
algorithms and this larger class of algorithms consists of exactly
the algorithms that are optimistic at every time step without
further restrictions. The conservative agent is the subclass of
algorithms that only switch hypothesis when the previous is
contradicted. The results for the conservative agent can be extended
to the liberal one, but we have to omit that here for space reasons.

\section{Stochastic Environments}\label{finitesto}

A stochastic hypothesis may never become completely inconsistent in
the sense of assigning zero probability to the observed sequence
while still assigning very different probabilities than the true
environment. Therefore, we exclude based on a threshold for the
probability assigned to the generated history. Unlike in the
deterministic case, a hypothesis can cease to be the optimistic one
without having been excluded. We, therefore, only consider an
algorithm that reevaluates its optimistic hypothesis at every time
step. Algorithm 2 specifies the procedure and Theorem \ref{thm:prob}
states that it is asymptotically optimal.

\begin{algorithm}[!h]
\caption{Optimistic Agent ($\po$) with Stochastic Finite
Class}\label{algprob}
\begin{algorithmic}[1]\REQUIRE Finite class of stochastic
environments $\M_1\equiv\M$, threshold $\z\in (0,1)$
 \STATE $t=1$
 \REPEAT
  \STATE $(\pi^*,\nu^*)=\argmax_{\pi,\nu\in\M_t}V^\pi_\nu(h_{t-1})$
  \STATE $a_t=\pi^*(h_{t-1})$
  \STATE Perceive $o_t r_t$ from environment $\mu$
  \STATE $h_{t}\leftarrow h_{t-1} a_t o_t r_t$
  \STATE $t\leftarrow t+1$
  \STATE $\M_t:=\{\nu\in\M_{t-1}:\frac{\nu(h_t|a_{1:t})}{\max_{\tilde{\nu}\in
\M} \tilde{\nu}(h_t|a_{1:t})}\geq\z\}$ \UNTIL the end of time

\end{algorithmic}
\end{algorithm}

\begin{theorem}[Optimality, Finite Stochastic Class]\label{thm:prob}
Define $\po$ by using Algorithm 2 with any threshold $\z\in(0,1)$
and a finite class $\M$ of stochastic environments containing the
true environment $\mu$, then with probability $1-\z|\M-1|$ there
exists, for every $\eps>0$, a number $T<\infty$ such that
$$
  V^\po_\mu(h_t) \;>\; \max_\pi V^\pi_\mu(h_t)-\eps\ \forall t\geq
  T.
$$
\end{theorem}

We borrow some techniques from \cite{Hutter09} that introduced a
``merging of opinions" result that generalized the classical theorem
by \cite{blackwell}. The classical result says that it is sufficient
that the true measure (over infinite sequences) is absolutely
continuous with respect to a chosen a priori distribution to
guarantee that they will almost surely merge in the sense of total
variation distance. The generalized version is given in Lemma
\ref{lemma:m}. When we combine a policy $\pi$ with an environment
$\nu$ by letting the actions be taken by the policy, we have defined
a measure, denoted by $\nu(\cdot|\pi)$, on the space of infinite
sequences from a finite alphabet.  We denote such a sample sequence
by $\w$ and the $a$:th to $b$:th elements of $\w$ by $\w_{a:b}$. The
$\sigma$-algebra is generated by the cylinder sets
$\Gamma_{y_{1:t}}:=\{\w|\w_{1:t}=y_{1:t}\}$ and a measure is
determined by its values on those sets. To simplify notation in the
next lemmas we will write $P(\cdot)=\nu(\cdot|\pi)$, meaning that
$P(\w_{1:t})=\nu(h_t| a_{1:t})$ where $\w_j=o_jr_j$ and
$a_j=\pi(h_{j-1})$. Furthermore, $\nu(\cdot|h_t,\pi)=P(\cdot|h_t)$.

\begin{definition}[Total Variation Distance]
The total variation distance between two measures (on infinite
sequences $\w$ of elements from a finite alphabet) $P$ and $Q$ is
defined to be
$$
  d(P,Q) \;=\; \sup_A |P(A)-Q(A)|
$$
where $A$ is in the previously specified $\sigma$-algebra generated
by the cylinder sets.
\end{definition}

The results from \cite{Hutter09} are based on the fact that
$Z_t=\frac{Q(\w_{1:t})}{P(\w_{1:t})}$ is a martingale sequence if
$P$ is the true measure and therefore converges with $P$ probability
$1$ \cite{doob}. The crucial question is if the limit is strictly
positive or not. The following lemma shows that with $P$ probability
$1$ we are either in the case where the limit is $0$ or in the case
where $d(P(\cdot|\w_{1:t}),Q(\cdot|\w_{1:t}))\to 0$. We say that the
environments $\nu_1$ and $\nu_2$ {\emph merge} under $\pi$ if
$d(\nu_1(\cdot|\pi),\nu_2(\cdot|\pi))\to 0$.

\begin{lemma}[Generalized merging of opinions \cite{Hutter09}]\label{lemma:m}
For any measures $P$ and $Q$ it holds that
$P(\Omega^\circ\cup\bar\Omega)=1$ where
$$
  \Omega^\circ:= \{\w:\frac{Q(\w_{1:t})}{P(\w_{1:t})}\to 0\} ~~~\text{and}~~~
  \bar\Omega:=\{\w:d(P(\cdot|\w_{1:t}),Q(\cdot|\w_{1:t}))\to 0\}
$$
\end{lemma}

\begin{lemma}[Value convergence for merging environments]\label{lemma:V}
Given a policy $\pi$ and environments $\mu$ and $\nu$ it follows
that
$$
  |V^\pi_\mu(h_t)-V^\pi_\nu(h_t)| \;\leq\;
  \frac{1}{1-\g} d(\mu(\cdot | h_t,\pi),\nu(\cdot | h_t,\pi)).
$$
\end{lemma}

\begin{proof}
The lemma follows from the general inequality
$$
  \big|\E_P(f)-\E_Q(f)\big| \;\leq\; \sup|f|\cdot \sup_A\big|P(A)-Q(A)\big|
$$
by inserting $f:=R(\w_{t:\infty})$ and $P=\mu(\cdot|h_t,\pi)$ and
$Q=\nu(\cdot|h_t,\pi)$, and using $0\leq f\leq 1/(1-\g)$.
\end{proof}

The following lemma replaces the property for deterministic
environments that either they are consistent indefinitely or the
probability of the generated history becomes $0$.

\begin{lemma}[Merging of environments]\label{lemma:env}
Suppose we are given two environments $\mu$ (the true one) and $\nu$
and a policy $\pi$ (defined e.g.\ by Algorithm 2). Let
$P(\cdot)=\mu(\cdot|\pi)$ and $Q(\cdot)=\nu(\cdot|\pi)$. Then with
$P$ probability $1$ we have that
$$
  \lim_{t\to\infty}\frac{Q(\w_{1:t})}{P(\w_{1:t})}= 0 ~~~\text{or}~~~
  \lim_{t\to\infty}|V^\pi_\mu(h_t)-V^\pi_\nu(h_t)|= 0.
$$
\end{lemma}

\begin{proof}
This follows from a combination of Lemma \ref{lemma:m} and Lemma
\ref{lemma:V}.
\end{proof}

The next lemma tells us what happens after all the environments that
will be removed have been removed but we state it as if this was
time $t=0$ for notational simplicity.

\begin{lemma}[Optimism is nearly optimal]\label{optpol}
Suppose that we have a (finite or infinite) class of (possibly)
stochastic environments $\M$ containing the true environment $\mu$.
Also suppose that none of these environments are excluded at any
time by Algorithm 2 ($\po$) during an infinite history $h$ that has
been generated by running $\po$ in $\mu$. Given $\varepsilon>0$
there is $\tilde{\varepsilon}>0$ such that
$$V^{\po}_\mu(\epstr)\geq\max_{\pi}
V^{\pi}_\mu(\epstr)-\varepsilon$$ if
$$|V^\po_{\nu_1}(h_t)-V^\po_{\nu_2}(h_t)|<\tilde{\varepsilon}\ \forall t, \forall
\nu_1,\nu_2\in\M.$$
\end{lemma}

\begin{proof}{\bf (Theorem \ref{thm:prob})}
Given a policy $\pi$, let $P(\cdot)=\mu(\cdot|\pi)$ where $\mu\in\M$
is the true environment and $Q=\nu(\cdot|\pi)$ where $\nu\in\M$. Let
the outcome sequence (the sequence $(o_1r_1),(o_2r_2),...$) be
denoted by $\w$. It follows from Doob's Martingale inequality
\cite{doob} that for all $\z\in (0,1)$
$$
  P(\sup_t\frac{Q(\w_{1:t})}{P(\w_{1:t})}\geq 1/\z) \;\leq\; \z\ \;,
~~~\text{ which implies }~~~
P(\inf_t\frac{P(\w_{1:t})}{Q(\w_{1:t})}\leq \z) \;\leq\; \z.
$$
This proves, using a union bound, that the probability of Algorithm
2 ever excluding the true environment is less than $\z|\M-1|$.

The limits $\frac{\nu(h_t|\po)}{\mu(h_t|\po)}$ converge almost
surely as argued before using the Martingale convergence theorem.
Lemma \ref{lemma:env} tells us that any given environment (with
probability one) is eventually excluded or is permanently included
and merge with the true one under $\po$. The remaining environments
does, according to (and in the sense of) Lemma \ref{lemma:env},
merge with the true environment. Lemma \ref{lemma:V} tells us that
the difference between value functions (for the same policy) of
merging environments converges to zero. Since there are finitely
many environments and the ones that remain indefinitely in $\M_t$
merge with the true environment under $\po$, there is for every
$\tilde{\eps}>0$ a $T$ such that when following $\po$, it holds for
all $t\geq T$ that
$$|V^\po_{\nu_1}(h_t)-V^\po_{\nu_2}(h_t)|<\tilde{\varepsilon}\
\forall \nu_1,\nu_2\in\M_t.$$ The proof is concluded by Lemma
\ref{optpol} in the case where the true environment remains
indefinitely included which happens with probability $\z|\M-1|$.
\end{proof}

\section{Compact Classes}\label{compact}

In this section we discuss infinite but compact classes of
stochastic environments. First note that without further
assumptions, asymptotic optimality can be impossible to achieve,
even for countably infinite deterministic environments
\cite{Lattimore11}. Here we consider classes that are compact with
respect to the total variation distance, or more precisely with
respect to $$\tilde{d}(\nu_1,\nu_2)=\max_{h,\pi}
d(\nu_1(\cdot|h,\pi),\nu_2(\cdot|h,\pi))$$ where $d$ is total
variation distance from Section \ref{finitesto}. An example is the
class of Markov Decision Processes (or POMDPs) with a certain number
of states. Algorithm 2 does need modification to achieve asymptotic
optimality in the compact case. An alternative to modifying the
algorithm is to be satisfied with reaching optimality within a
pre-chosen $\eps>0$. This can be achieved by first choosing a finite
covering of $\M$ with balls of total variation radius less than
$\eps(1-\gamma)$ and use Algorithm 2 with the centers of these
balls. To have an algorithm that for any $\eps>0$ eventually
achieves optimality within $\eps$ is a more demanding task. This is
because we need to be able to say that the true environment will
remain indefinitely in the considered class with a given confidence.
For this purpose we introduce a confidence radius inspired by MDP
solving algorithms like MBIE \cite{strehl05} and UCRL \cite{auer06}.
We still use the notation $\M_t$ as in Algorithm 2 and we define
Algorithm 3 based on replacing it with a larger $\tilde{\M}_t$. If
we do not do this the true environment is likely to be excluded.

\begin{definition}[Confidence radius]
We denote all environments within $r^\z_t$ from $\M_t$ by
$$\tilde{\M}_t:=\{\nu\in\M\ |\ \exists \tilde{\nu}\in\M_t:
\tilde{d}(\tilde{\nu},\nu)\leq r_t^\z\}.$$ Given $\z>0$ we say that
$r^\z_t(h_t)$ is a $p$-confidence radius sequence if $r^\z_t(h_t)\to
0$ almost surely and if the true environment is in $\tilde{M_t}$ for
all $t$ with probability $p$.
\end{definition}

\begin{definition}[Algorithm 3]
Given a class of environments $\M$ that is compact in the total
variation distance we define Algorithm 3 as being Algorithm 2 with
$\M_t$ replaced by $\tilde{\M}_t$
\end{definition}

\begin{definition}[Radon-Nikodym differentiable class]
Suppose that the class $\M$ is such that if $\mu\in\M$ is the true
environment, then for any policy $\pi$ it holds with probability one
that for all $\nu\in\M$,
$X_{t,\nu}:=\frac{\nu(h_t|\pi)}{\mu(h_t|\pi)}$ converges as
$t\to\infty$ to some random variables $X_\nu$. We call such a class
Radon-Nikodym (RN) differentiable.  If the property holds with
respect to a specific policy $\pi$ we say that the class is
RN-differentiable with respect to $\pi$.
\end{definition}

\begin{remark}
Every countable class is RN-differentiable and so is the class of
MDPs with a certain number of states. The MBIE \cite{strehl05} and
UCRL \cite{auer06} algorithms are based on the fact that one can
define confidence radiuses for MDPs, though their bounds need
separate intervals for each state-action pair depending on the
number of visits. For an ergodic MDP all state-action pairs will
almost surely be seen infinitely often and the max length of those
intervals will tend to zero. Therefore, one can define a radius
based on this maximum length or, alternatively, one can easily allow
Algorithm 3 to run with such rectangular sets instead.
\end{remark}

\begin{theorem}[Optimality, Compact Stochastic Class]\label{thm:three}
Suppose we use Algorithm 3 with threshold $\z\in (0,1)$,  a compact
(in total variation) RN-differentiable class (with respect to $\po$
is enough) $\M$ of stochastic environments and a $p$-confidence
radius sequence $r_t^\z$ for $\M$. Denote the resulting policy by
$\po$. If the true environment $\mu$ is in $\M$, then with
probability $p$ there is, for every $\eps>0$, a tim e $T<\infty$
such that
$$
  V^\po_\mu(h_t) \;\geq\; \max_\pi V^\pi_\mu(h_t)-\eps\ \forall
  t\geq T.
$$
\end{theorem}

\begin{lemma}[Uniform exclusion]\label{lemma:uexcl}
Let $Q_\nu(\cdot)=\nu(\cdot|\po)$ and $P(\cdot)=\mu(\cdot|\po)$
where $\mu$ is the true environment and $\po$ the policy defined by
Algorithm 3. For any outcome sequence $\w$, let
$$
  \M^0(\w) \;:=\; \{\nu\ |\ \frac{Q_\nu(\w_{1:t})}{P(\w_{1:t})}\to 0\}.
$$
For any closed subset of $\M^0(\w)$ and for every $\z>0$, there is
$T<\infty$ such that for every $\nu$ in this subset there is $t\leq
T$ such that $\frac{Q_{\nu}(\w_{1:t})}{P(\w_{1:t})}<\z.$
\end{lemma}

\begin{proof}
Since $\M$ is compact and the subset in question is closed it
follows that it is also compact. Using the Arzel\`a-Ascoli Theorem
\cite{rudin} we conclude that there is a subsequence $t_k$ such that
$Z^\nu_k:=\min\{1,\frac{Q_{\nu}(\w_{1:t_k})}{P(\w_{1:t_k})}\}$
converges uniformly to $0$ on $\M^0$ which means that there is $t_k$
such that $Z_k^\nu<\z$ for all $\nu\in\M^0$ and we can let
$t=T=t_k$.
\end{proof}

\begin{proof}{\bf (Theorem \ref{thm:three})}
The strategy is to use that all environment that will be excluded
and does not lie within a certain distance of some environment that
merges with the true one, will be excluded after a certain finite
time. Then we can say that the remaining environments' value
functions differ at most by a certain amount and we can apply Lemma
\ref{optpol}.

We can with probability one say that for each $\nu\in\M$, it will
hold that $Z_t=\frac{\nu(h_t|\po)}{\mu(h_t|\po)}$ converges and each
environment will be in $\M^0=\{\nu\in\M\ |\ Z_t\to 0\}$ or
$\bar{\M}=\{\nu\ |\ d(\nu(\cdot|h_t,\po),\mu(\cdot|h_t,\po))\to
0\}$. $\bar{\M}$ is compact (in the total variation distance
topology) since it is a closed subset (again in the topology defined
by $\tilde{d}$) of the compact set $\M$.

For any $\tilde{\eps}_1>0$ we can do the following: For each
$\nu\in\M$, consider a total variation ball of radius $2\delta$
where $\delta=(1-\gamma)\tilde{\eps}_1/4$. Note that
$|V^\po_\nu(h_t)-V^\po_{\nu'}(h_t)|<\tilde{\eps}_1/2$ for all $t$
whenever $\tilde{d}(\nu,\nu')<2\delta$. The collection of these
balls induces an open cover of the compact set $\M$ and it follows
that there is a finite subcover. Consider the balls in this finite
cover that intersect with $\bar{\M}$. Let $\A$ be the union of these
finitely many open balls. Let $\B=\M\setminus \A$. $\B$ is then a
closed subset of $\M^0$. We want to say that there is a finite time
after which all environments in $\B$ will have been excluded from
$\tilde{\M}_t$. This happens if $\tilde{\B}$, defined as the union
of the closed balls of radius $r_t^\z$ at every point in $\B$, has
been excluded from $\M_t$. If $t$ is large enough for
$r_t^\z<\delta$, then $\B$ is also a closed subset of $\M_0$. Lemma
\ref{lemma:uexcl} tells us that all of the environments in
$\tilde{\B}$ will have been excluded from $\M_t$ after a finite
amount of time $T_1$ and, therefore, all the environments in $\B$
will have been excluded from $\tilde{\M}_t$. Thus
$\tilde{\M}_t\subset\A\ \forall t\geq T_1$ and in particular the
optimistic hypothesis $\nu^*$ will be in $\A$ when $t\geq T_1$. Let
$\nu^* (=\nu^*_t)$ be the optimistic hypothesis at time $t\geq T_1$
and $\pi^* (=\pi^*_t)$ the optimistic policy.

Each parameter in $\A$ (and in particular $\nu^*$) lies within
$\delta$ of a ball with center $\nu$ which lies within $\delta$ of a
point $\tilde{\nu}\in \bar{\M}$. Hence
$\tilde{d}(\nu^*,\tilde{\nu})<2\delta$ and
$|V^\po_{\nu^*}(h_t)-V^\po_{\tilde{\nu}}(h_t)|<\tilde{\eps}_1/2$.

Due to the uniform merging of environments (under $\po$) on
$\bar{\M}$, there is $T_2\geq T_1$ such that
$|V^\po_{\nu_1}(h_t)-V^\po_{\nu_2}(h_t)|<\tilde{\eps}_1/2\ \forall
\nu_1,\nu_2\in\bar{\M}\ \forall t\geq T_2$. We conclude that
$|V^\po_{\nu_1}(h_t)-V^\po_{\nu_2}(h_t)|<\tilde{\eps}_1\ \forall
\nu_1,\nu_2\in\A\ \forall t\geq T_2$ and since
$\tilde{\M}_t\subset\A$
$$|V^\po_{\nu_1}(h_t)-V^\po_{\nu_2}(h_t)|<\tilde{\eps}_1\
\forall \nu_1,\nu_2\in\tilde{\M}_t\ \forall t\geq T_2.$$ From Lemma
\ref{optpol} we know that if we picked $\tilde{\eps}_1$ small enough
we know that for $t\geq T_2$, $V^{\po}_{\nu^*}(h_t)\geq
V^{\pi}_{\nu}(h_t)-\eps/2$ for all $\pi\in\Pi,\nu\in\tilde{\M}_t$.
Furthermore, by picking $\tilde{\eps}_1$ sufficiently small we can,
for $t\geq T_2$, ensure that there is $\tilde{\nu}\in\tilde{\M}_t$
such that $|V^{\po}_{\tilde{\nu}}(h_t)-V^{\po}_{\mu}(h_t)|<\eps/2$.
Given that the true environment remains indefinitely in
$\tilde{M}_t$, which happens with at least probability $p$, it
follows that
$$
  V^{\po}_\mu(h_t) \;\geq\; \max_{\pi} V^{\pi}_\mu(h_t)-\eps\;\; \forall t\geq T_2.
\vspace{-4ex}$$\vspace{-4ex}
\end{proof}

\section{Conclusions}\label{concl}

We introduced optimistic agents for finite and compact classes of
arbitrary environments and proved asymptotic optimality. In the
deterministic case we also bound the number of time steps for which
the value of following the algorithm is more than a certain amount
lower than optimal. Future work includes investigating finite-error
bounds for classes of stochastic environments.

\paradot{Acknowledgement} This work was supported by ARC grant
DP120100950. The authors are grateful for feedback from Tor
Lattimore and Wen Shao.



\begin{thebibliography}{10}

\bibitem[AO06]{auer06}
P.~Auer and R.~Ortner.
\newblock Logarithmic online regret bounds for undiscounted reinforcement
  learning.
\newblock In {\em Proceedings of NIPS'2006}, pages 49--56, 2006.

\bibitem[BD62]{blackwell}
D.~Blackwell and L.~Dubins.
\newblock {Merging of Opinions with Increasing Information}.
\newblock {\em The Annals of Mathematical Statistics}, 33(3):882--886, 1962.

\bibitem[Doo53]{doob}
J.~Doob.
\newblock {\em Stochastic processes}.
\newblock Wiley, New York, NY, 1953.

\bibitem[EDKM05]{DKM05}
E.~Even-Dar, S.~Kakade, and Y.~Mansour.
\newblock Reinforcement learning in pomdps without resets.
\newblock In {\em Proceedings of IJCAI-05}, pages 690--695, 2005.

\bibitem[Hut05]{MH05}
M.~Hutter.
\newblock {\em Universal Articial Intelligence: Sequential Decisions based on
  Algorithmic Probability}.
\newblock Springer, Berlin, 2005.

\bibitem[Hut09]{Hutter09}
M.~Hutter.
\newblock Discrete {MDL} predicts in total variation.
\newblock In {\em Advances in Neural Information Processing Systems 22:
  (NIPS'2009)}, pages 817--825, 2009.

\bibitem[KS98]{KS98}
M.~J. Kearns and S.~Singh.
\newblock Near-optimal reinforcement learning in polynomial time.
\newblock In {\em Proceedings of the $15^{nd}$ International Conference on
  Machine Learning (ICML'1998)}, pages 260--268, 1998.

\bibitem[LH11a]{Lattimore11}
T.~Lattimore and M.~Hutter.
\newblock Asymptotically optimal agents.
\newblock In {\em Proc. of Algorithmic Learning Theory (ALT'2011)}, volume 6925
  of {\em Lecture Notes in Computer Science}, pages 368--382. Springer, 2011.

\bibitem[LH11b]{Tor11}
T.~Lattimore and M.~Hutter.
\newblock Time consistent discounting.
\newblock In {\em Proc. 22nd International Conf. on Algorithmic Learning Theory
  ({ALT'11})}, volume 6925 of {\em LNAI}, pages 383--397, Espoo, Finland, 2011.
  Springer, Berlin.

\bibitem[LH12]{Tor12}
T.~Lattimore and M.~Hutter.
\newblock {PAC} bounds for discounted {MDP}s.
\newblock In {\em Proc. 23rd International Conf. on Algorithmic Learning Theory
  ({ALT'12})}, volume 7568 of {\em LNAI}, Lyon, France, 2012. Springer, Berlin.

\bibitem[MMR11]{MR11}
O.-A. Maillard, R.~Munos, and D.~Ryabko.
\newblock Selecting the state-representation in reinforcement learning.
\newblock In {\em Advances in Neural Information Processing Systems 24
  (NIPS'2011)}, pages 2627--2635, 2011.

\bibitem[Ors10]{Orseau10}
L.~Orseau.
\newblock Optimality issues of universal greedy agents with static priors.
\newblock In {\em Proc. of Algorithmic Learning Theory, (ALT'2010)}, volume
  6331 of {\em Lecture Notes in Computer Science}, pages 345--359. Springer,
  2010.

\bibitem[RH08]{RH08}
D.~Ryabko and M.~Hutter.
\newblock On the possibility of learning in reactive environments with
  arbitrary dependence.
\newblock {\em Theor. C.S.}, 405(3):274--284, 2008.

\bibitem[RN10]{RN10}
S.~J. Russell and P.~Norvig.
\newblock {\em Artificial Intelligence: A Modern Approach}.
\newblock Prentice Hall, Englewood Cliffs, NJ, $3^{nd}$ edition, 2010.

\bibitem[Rud76]{rudin}
W.~Rudin.
\newblock {\em Principles of mathematical analysis}.
\newblock McGraw-Hill, 1976.

\bibitem[SB98]{SB98}
R.~Sutton and A.~Barto.
\newblock {\em Reinforcement Learning}.
\newblock The MIT Press, 1998.

\bibitem[SL05]{strehl05}
A.~Strehl and M.~Littman.
\newblock A theoretical analysis of model-based interval estimation.
\newblock In {\em Proceedings of ICML 2005}, pages 856--863, 2005.

\end{thebibliography}
\end{document}